\documentclass[twoside]{article}

\usepackage[accepted]{aistats2019}

\usepackage[utf8]{inputenc} % allow utf-8 input
\usepackage[T1]{fontenc}    % use 8-bit T1 fonts\\
\usepackage{url}            % simple URL typesetting
\usepackage{booktabs}       % professional-quality tables
\usepackage{amsfonts}       % blackboard math symbols
\usepackage{nicefrac}       % compact symbols for 1/2, etc.
\usepackage{microtype}      % microtypography
\usepackage{mathtools}% Loads amsmath
\usepackage{hyperref}      % hyperlinks
\usepackage{amsmath}
\usepackage{amsthm}
\usepackage{amsfonts}
\usepackage{amssymb}
\usepackage{graphicx}
\usepackage{booktabs}
\usepackage{placeins}
\usepackage{multirow}
\usepackage[ruled, linesnumbered]{algorithm2e}
\usepackage{epstopdf}
\usepackage{color}
\usepackage{subcaption}
\usepackage[T1]{fontenc}
\usepackage{lmodern}

\usepackage[round]{natbib}

\newtheorem{theorem}{Theorem}[section]

\newtheorem{prop}{Proposition}
\theoremstyle{definition}
\newtheorem{definition}{Definition}[section]

\newcommand{\lov}{Lov{\'a}sz}
\def\vec#1{\mbox{\bf #1}}
\def\mat#1{\mbox{\bf #1}}%% usage: \mat{W}.
\runningtitle{Lov{\'a}sz Convolutional Networks}
\begin{document}


% st\newcommand{\lov}{Lov{\'a}sz}
\def\vec#1{\mbox{\bf #1}}
%% Matrix
\def\mat#1{\mbox{\bf #1}}%% usage: \mat{W}.yle will print as headings an error message. Use the following

\twocolumn[
%\rightline{ {\large \textcolor{red}{[Accepted at AISTATS 2019]}}}
\aistatstitle{Lov{\'a}sz Convolutional Networks}

\aistatsauthor{
		Prateek Yadav$^1$ \quad
		Madhav Nimishakavi$^1$\quad
		Naganand Yadati$^1$\\
		\textbf{Shikhar Vashishth}$^1$\quad
		\textbf{Arun Rajkumar}$^2$\quad
		\textbf{Partha Talukdar}$^1$\\
		$^1$ Indian Institute of Science\\
		$^2$ Conduent Labs, India \\
		{\tt \small \{prateekyadav,madhav,naganand,shikhar,ppt\}@iisc.ac.in} \\
		{\tt \small vdrn485@gmail.com}
}

\aistatsaddress{ } ]

\begin{abstract}
    Semi-supervised learning on graph structured data has received significant attention with the recent introduction of Graph Convolution Networks (GCN). While traditional methods have focused on optimizing a loss augmented with  Laplacian regularization framework, GCNs perform an implicit Laplacian type regularization to capture local graph structure. In this work, we propose  \emph{\lov{} Convolutional Network} (LCNs) which are capable of incorporating global graph properties. LCNs achieve this by utilizing \lov{}'s orthonormal embeddings of the nodes. We analyse local and global properties of graphs and demonstrate settings where LCNs tend to work better than GCNs. We validate the proposed method on standard random graph models such as stochastic block models (SBM) and certain community structure based graphs where LCNs outperform GCNs and learn more intuitive embeddings. We also perform extensive binary and multi-class classification experiments on real world datasets to demonstrate LCN's effectiveness. In addition to simple graphs, we also demonstrate the use of LCNs on hyper-graphs by identifying settings where they are expected to work better than GCNs.
\end{abstract}

% !TeX spellcheck = en_US
\section{Introduction}
\vspace{-3mm}

Learning on structured data has received significant interest in recent years \citep{getoor2007introduction,subramanya2014graph}. Graphs are ubiquitous, several real world data-sets can be naturally represented as graphs; knowledge graphs \citep{yago07, dbpedia07, freebase08}, protein interaction graphs \citep{protein17}, social network graphs \citep{sn10, sn_www10, sn12}, citation networks \citep{citeseer98, cn03, ccnd08} to name a few. These graphs typically have a large number of nodes and manually labeling them as belonging to a certain class is often prohibitive in terms of resources needed. A common approach is to pose the classification problem as a semi-supervised graph transduction problem where one wishes to label all the nodes of a graph using the labels of a small subset of nodes.

Recent approaches to the graph transduction problem rely on the assumption that the labels of nodes are  related to the structure of the graph. A common approach is to use the Laplacian matrix associated with a graph as form of a structural regularizer for the learning problem. While the Laplacian regularization is done explicitly in \citep{shivani_icml06, sslintroicml03, Zhou2004, sslintro06, planetoid_icml16}, more recent deep learning based Graph Convolution Network (GCN) approaches do an implicit Laplacian type regularization \citep{dcnn_nips16, gcniclr17, co_self_gcn_aaai18, dual_gcn_www18}. While these traditional methods work reasonably well for several real world problems, our extensive experiments show that they may not be the best methods for tasks involving communities and there is a scope for significant improvement in such cases.
%\begin{figure*}[tbh]
%	\begin{subfigure}{.5\textwidth}
%		\begin{subfigure}{.5\textwidth}
%			\centering
%			\includegraphics[width=0.8\linewidth]{./figures/embed_aistats/sbm5545_kipf}
%%			\caption{1a}
%			\label{fig1:kipf55}
%		\end{subfigure}%
%		\begin{subfigure}{.5\textwidth}
%			\centering
%			\includegraphics[width=0.8\linewidth]{./figures/embed_aistats/sbm5545_LCN}
%%			\caption{1b}
%			\label{fig1:lov55}
%		\end{subfigure}
%		\caption{\small Binary classification }
%		\label{fig1:2d-embed}
%	\end{subfigure}
%	\begin{subfigure}{.5\textwidth}
%		\begin{subfigure}{.5\textwidth}
%			\centering
%			\includegraphics[width=0.8\linewidth]{./figures/embed_aistats/sbm6444_kipf_c}
%%			\caption{1a}
%			\label{fig2:kipf644}
%		\end{subfigure}%
%		\begin{subfigure}{.5\textwidth}
%			\centering
%			\includegraphics[width=0.8\linewidth]{./figures/embed_aistats/sbm6444_LCN_c}
%%			\caption{1b}
%			\label{fig2:lov644}
%		\end{subfigure}
%		\caption{\small 3-class classification}
%		\label{fig2:3d-embed}
%	\end{subfigure}
%	\caption{\small Node embeddings for SBM Experiments}
%\end{figure*}

\begin{figure*}[tbh]
	\begin{subfigure}{.5\textwidth}
		\centering
		\includegraphics[width=0.95\linewidth]{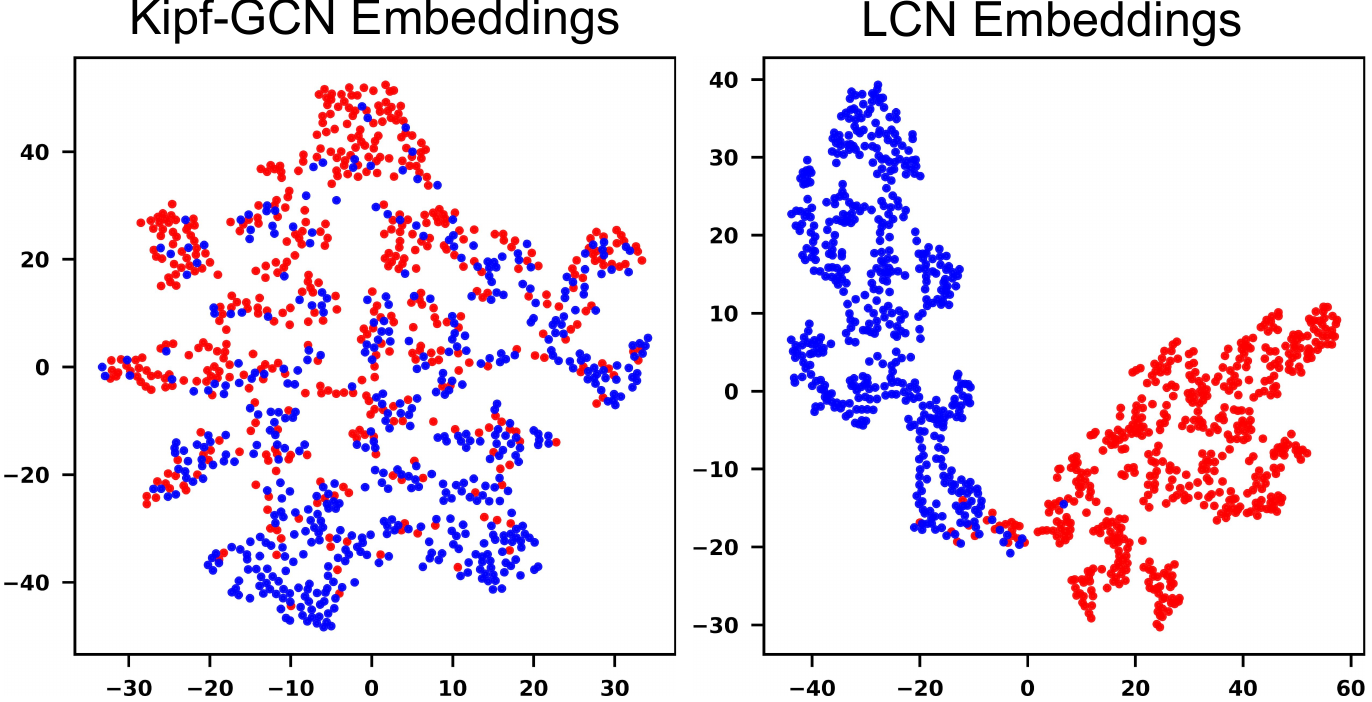}
		\caption{\small Binary classification }
		\label{fig1:2d-embed}
	\end{subfigure}
	\begin{subfigure}{.5\textwidth}
		\centering
		\includegraphics[width=0.95\linewidth]{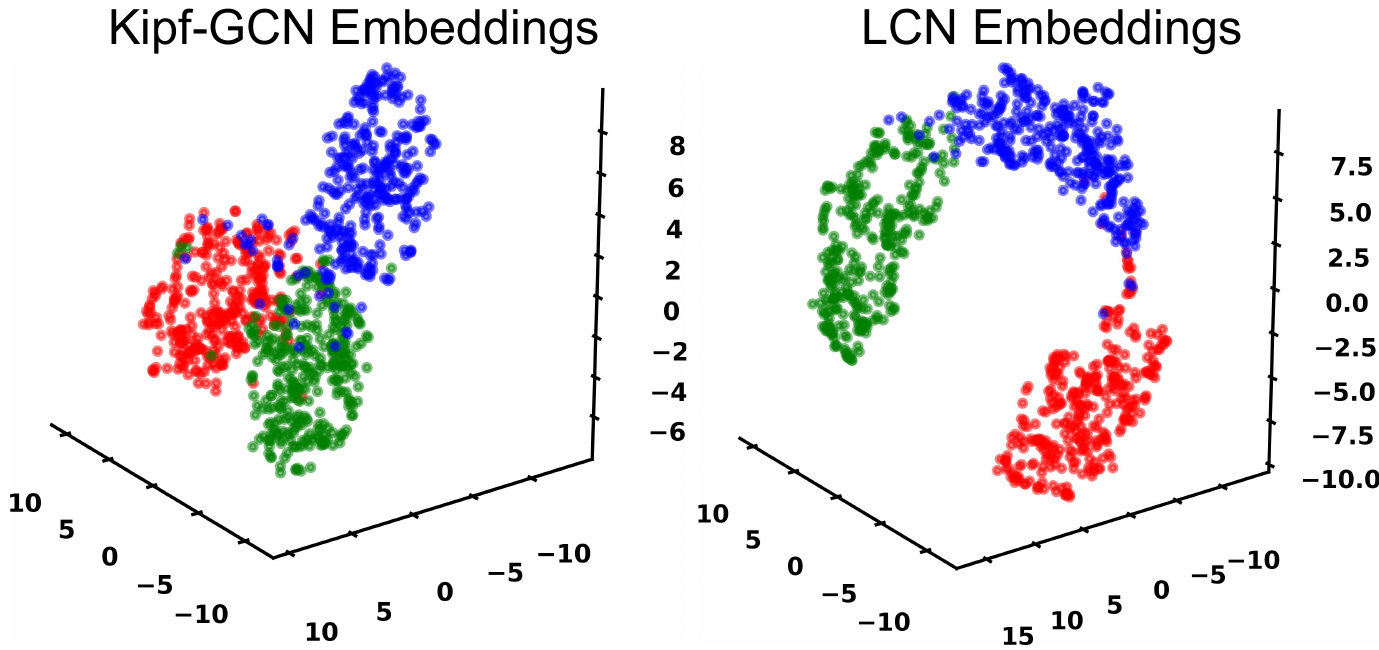}
		\caption{\small 3-class classification}
		\label{fig2:3d-embed}
	\end{subfigure}
	\caption{\small Node embeddings for SBM Experiments}
\end{figure*}

In this work, we propose a graph convolutional network based approach to solve the semi-supervised learning problem on graphs that typically have a community structure. An extensively studied model for communities is the Stochastic block model (SBM) which is a random graph model where the nodes of a graph exhibit community structure i.e., the nodes belonging to same community have a larger probability of having an edge between them than those in different communities. In this work, we propose the \emph{\lov{} Convolutional Network} (LCN) which, instead of the traditional Laplacian, uses the embeddings of nodes that arise from \lov{}'s orthogonal representations as an implicit regularizer. The \lov{} regularization, as we will see, is tightly coupled to the coloring of the complement graph of a given graph and hence often produce remarkably superior embeddings than those obtained using the Laplacian regularization for graphs which have a community structure. Intuitively, the optimal coloring of the complement of a graph can be viewed as a way to associate same color to nodes belonging to a same community. As \lov{} embeddings also tend to embed nodes with same colors to similar points in Euclidean space, the proposed model performs much well in practice. Figure \ref{fig1:2d-embed} and Figure \ref{fig2:3d-embed} illustrate this phenomenon using examples for a binary and a three class classification problem where the graph is generated using a stochastic block model. As can be seen, the average distance between embeddings learnt using the LCN is much better than using traditional graph based convolution networks.
%\begin{figure}[tbh]
%	\centering
%	\begin{tabular}{cccc}
%		\begin{minipage}[b]{0.25\textwidth}
%			\centering
%			\includegraphics[width=\textwidth]{./figures/embed/sbm5545_kipf}			
%		\end{minipage}
%		\begin{minipage}[b]{0.25\textwidth}
%			\centering
%			\includegraphics[width=\textwidth]{./figures/embed/sbm5545_lov}
%		\end{minipage}		
%		\begin{minipage}[b]{0.25\textwidth}
%			\centering
%			\includegraphics[scale=0.28]{./figures/embed/sbm_6444_kipf_tsne}
%			{\small (c) Embeddings for SBM: Example 2}
%			\label{fig2}
%		\end{minipage}
%		\begin{minipage}[b]{0.25\textwidth}
%			\centering
%			\includegraphics[scale=0.25]{./figures/embed/sbm_6444_lov_tsne}
%			{\small (d) Embeddings for SBM: Example 2}
%			\label{fig2}
%		\end{minipage}
%	\end{tabular}
%\end{figure}
We make the following contributions in this work:
\begin{itemize}
\itemsep0em
    \item We propose the \emph{\lov{} Convolutional Network (LCN)} for the problem of semi-supervised learning. LCN combines the power of using the \lov{} embeddings with GCNs.
    \item We analyze various types of graphs and identify the classes of graphs where LCN  performs much better than existing methods. In particular, we demonstrate that by keeping the optimal coloring, a global property of the graph, fixed and increasing the number edges to the graph, LCNs outperforms traditional GCNs. 
    \item We carry out extensive experiments on both synthetic and real world datasets and show significant improvement using LCNs than state of the art algorithms for semi-supervised graph transduction. 
\end{itemize}
The codes for our model are provided as supplementary material.

% !TeX spellcheck = en_US
\section{Related Work}
\vspace{-3mm}
\label{sec:related}

The work that is most closely related to ours is \citep{Shivanna2015} which proposes a spectral regularized orthogonal embedding method for graph transduction (SPORE). While they use a \lov{} embedding based kernel for explicit regularization, the focus is on computing the embedding efficiently using a special purpose optimization routine. Our work on the other hand proposes a deep learning based \lov{} convolutional network which differs from the traditional loss plus explicit regularizer approach of  \citep{Shivanna2015} and our experimental results confirm that the proposed LCN approach performs significantly better than SPORE. The use of explicit Laplacian regularizer for semi-supervised learning problems on graphs has been explored in \citep{Ando2007,shivani_icml06} where the focus is to derive generalization bounds for learning on graphs.  However, as we will discuss in the sequel, there are  settings where \lov{} embeddings are more natural in capturing the global property of graphs than the Laplacian embeddings and this reflects in our experimental results as well.  More recently \citep{dual_gcn_www18} propose a dual convolution approach to capture global graph property using positive pointwise mutual information (PPMI). We differ from this approach in defining global property in terms of coloring of the complement graph as opposed to computing semantic similarity using random walks on the graph as done in \citep{dual_gcn_www18}. \lov{} based kernels for graphs have been explored in the context of other machine learning problems such as clustering etc. in \citep{Johansson2014}. \cite{Jethava2013} show an interesting connection between \lov{} $\vartheta$ function and one class SVMs. \\
There has been considerable amount of work in extending well established deep learning architectures for graphs. \cite{gcn_iclr14,gcn_arxiv15,gcn_nips15,chebnet_nips16} extend Convolutional Neural Networks (CNN) for graphs, while \cite{Jain2016} propose Recurrent Neural Networks (RNN) for graphs. \citet{gcniclr17} propose Graph Convolutional Networks which achieve promising results for the problem of semi-supervised classification on graphs. Most recently, a faster version of GCN, for inductive learning on graphs, has been proposed by \citep{Chen2018}. An extension to GCNs based on graph partition is proposed recently by \citep{liao2018graph}.
However, as we show in our experiments, there are several natural settings where the proposed LCN performs much better than the state of the art GCNs in various problems of interest.
% !TeX spellcheck = en_US
\section{Problem Setting and Preliminaries}
\vspace{-3mm}
\label{sec:prelim}

We work in the semi-supervised graph transduction setting where we are given a graph $G(V,E)$, where $V$ denotes the set of vertices with cardinality $n$ and $E$ is the edge set. We are given the labels  ($\{0,1\}$ in the case of binary classification) of a subset of nodes $(m < n)$ of $V$ and the goal is to predict the labels of the remaining nodes as accurately as possible.  Given a graph $G(V,E)$, $\alpha(G)$ denotes the maximum independence number of the graph i.e., the size of the set containing the maximum number of non-adjacent nodes in $G$. A coloring of $G$ corresponds to an assignment of colors to nodes of the graph such that no two nodes with the same color have an edge between them.  $\chi(G)$ denotes the chromatic number of $G$ which is the minimum number of colors needed to color $G$. We denote the complement of a graph by $\bar{G}(V,\bar{E})$. An edge $(u,v)$ is present in $\bar{G}$ if and only if it is not present in $G$.  It is easily seen that for any graph $G$, $\alpha(G) \le \chi(\bar{G})$. A clique is a fully connected graph which has edges between all pairs of nodes. We assume that there is a natural underlying manner in which the graph structure is related to the class labels.  In what follows, we recall certain classes of graphs and a certain type of graph embedding which will be of interest in the rest of the paper. \\ 
\textbf{SBM Graphs}:  The Stochastic Block Model (SBM) \citep{holland1983stochastic, condon1999algorithms} is a generative model for random graphs. They are a generalization of the Erdos-Renyi graphs where the edges between nodes of the same community are chosen with a certain probability ($p$) while the edges across communities are chosen with a certain other probability ($q$ where $q < p$). SBMs tend to have community structure and hence are used to model several applications including protein interactions, social network analysis and have been extensively studied in machine learning, statistics, theoretical computer science and network science literature. \\ 
\textbf{Perfect Graphs}: Perfect graphs are a class of graphs whose chromatic number $\chi(G)$ equals the size of the largest clique for every induced subgraph. Several important class of graphs including bipartite graphs, interval graphs, chordal graphs, caveman graphs etc. are all perfect graphs. We refer to \citep{ramirez2001perfect} for graph theoretical analysis of perfect graphs. Our interest in these graphs is due to the fact that the the chromatic number of these graphs can be computed in polynomial time \citep{lovasz2009characterization} and they coincide with the \lov{} $\vartheta$ number of the graph which we discuss next. \\ 
\begin{figure}[t]
	\centering
	\includegraphics[width=0.6\columnwidth]{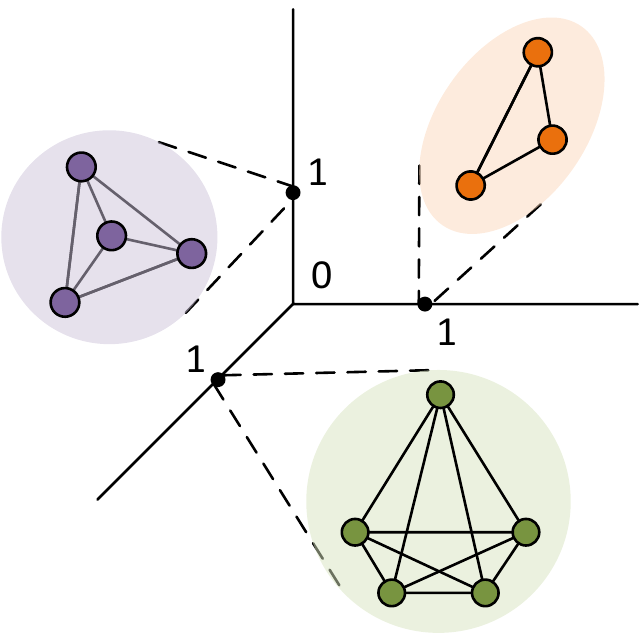}\\
	\caption{\small \lov{} embeddings for a graph consisting of set cliques are mapped orthogonal dimensions. Refer Section \ref{sec:prelim} for more details.}
	\label{fig:lovasz_axis}
\end{figure}
\textbf{\lov{} Embeddings}:  \lov{} \citep{Lovasz1979} introduced the concept of {\it orthogonal embedding} in the context of the problem of embedding a graph  $G = (V, E)$ on a unit sphere $\mathcal{S}^{d-1}$.
\begin{definition}[Orthogonal embedding \citep{Lovasz1979,lovasz99}]
An orthogonal embedding of a graph $G(V, E)$ with $|V| = n$, is a matrix $\mat{U} = [\vec{u}_1, \ldots, \vec{u}_n] \in \mathbb{R}^{d \times n}$ such that $\vec{u}_{i}^\top\vec{u}_{j} = 0$ whenever $(i,j) \notin E$ and $\vec{u}_{i} \in \mathcal{S}^{d-1} ~\forall i \in [n]$.
\end{definition}
Let $Lab(G)$ denotes the set of all possible orthogonal embeddings of the graph $G$, given by $Lab(G) = \{ \mat{U} | \mat{U} \text{ is an orthogonal embedding}\}$. The \lov{} theta function is defined as:
\begin{equation*}
 \vartheta(G) = \min_{\mat{U} \in Lab(G)} \min_{\vec{c} \in \mathcal{S}^{d-1}} \max_{i} (\vec{c}\top\vec{u}_i)^{-2}.
\end{equation*}
\vspace{-1.5mm} 

The famous sandwich theorem of \lov{}  \citep{Lovasz1979} states that $\alpha(G) \le \vartheta(G) \le \chi(\bar{G})$, where $\alpha(G)$ is the independence number of the graph and $\chi(\bar{G})$ is the chromatic number of the complement of $G$. Perfect graphs are of interest to us as both the above inequalities are equalities for them \citep{lovasz2009characterization}. \\
A few examples are helpful to gain intuition about the relation of \lov{} embeddings to community structures. For a complete graph, the complement can be colored using just one color, the \lov{} embeddings of all the nodes are trivial and in $1$-dimension. These embeddings are exactly the same as there are no orthogonal constraints imposed by the edges. As a generalization of this example  (Figure \ref{fig:lovasz_axis}), for a graph that is a disjoint union of $k$ cliques of possible variable number of nodes in each clique, the complement is a complete $k$ partite graph and hence can be colored using $k$ colors where each partition corresponds to a single color. It turns out the \lov{} embeddings for this graph are a set of orthonormal vectors in $\mathbb{R}^k$. In practice, the communities that occur are not exactly cliques i.e., not all edges in a community are connected to each other. However, the \lov{} embeddings still capture the necessary structure as we will see in our experiments. \\
\textbf{Graph convolutional networks (GCN):} GCNs \citep{gcniclr17} extend the idea of Convolutional Neural Networks (CNNs) for graphs. Let $G(V, E)$ be an undirected graph with adjacency matrix  $\mat{A}$ and let $\widetilde{\mat{A}} = \mat{A} + \mat{I}$ be the adjacency with added self-connections and $\widetilde{\mat{D}}_{ii} = \sum_{j}\widetilde{\mat{A}}_{ij}$. Let $\mat{X} \in \mathbb{R}^{n \times d}$ represent the input feature matrix of the nodes.
A simple two-layer GCN for the problem of semi-supervised node classification assumes the form :
% \begin{equation}
 $f(\mat{X}, \mat{A}) = \mathrm{softmax} ( \hat{\mat{A}} \text{~} \text{ReLU}(\hat{\mat{A}}\mat{X}\mat{W}^{(0)})\mat{W}^{(1)}).$
% \end{equation}
%\vspace{-1mm}
Where, $\hat{\mat{A}} = \widetilde{\mat{D}}^{-\frac{1}{2}}\widetilde{\mat{A}}\widetilde{\mat{D}}^{-\frac{1}{2}}$ , $\mat{W}^{(0)} \in \mathbb{R}^{d\times h}$ is an input-to-hidden weight matrix for a hidden layer with $h$ units and $\mat{W}^{(1)} \in \mathbb{R}^{h \times F}$ is hidden-to-output weight matrix. The softmax activation function, defined as $\mathrm{softmax}(x_i) = \frac{1}{Z}\exp(x_i)$ with $Z = \sum_i \exp(x_i)$ is applied row-wise. 

For semi-supervised multi-class classification, cross-entropy loss over the labeled examples is given by
\begin{equation}
\label{eqn:multi-loss}
 \mathcal{L} = \sum_{l \in \mathcal{Y}_L}\sum_{f=1}^{F} Y_{lf} \text{ ln } Z_{lf},
\end{equation}
where, $\mathcal{Y}_L$ is the set of labeled nodes. The weights $\mat{W}^{(0)}$ and $\mat{W}^{(1)}$ are learnt using gradient descent.

\section{Motivating Example}
\vspace{-3mm}
\label{sec:motivation}

%\begin{figure*}[t]
%	\centering
%	\begin{tabular}{ccc}
%		\centering
%		\begin{minipage}{0.40\hsize}	
%			\centering
%			\includegraphics[scale=0.35]{figures/lovasz_axis-crop}\\
%			\caption{\small \lov{} embeddings for clique (up), disjoint union of cliques (down). Refer Section \ref{sec:prelim} for more details.}
%			\label{fig:lovasz_axis}
%		\end{minipage}
%		\begin{minipage}{0.05\hsize}		 
%		\end{minipage}
%
%		\begin{minipage}{0.48\hsize}	
%			\centering		
%			\includegraphics[scale=0.27]{figures/bipartite_color}\\
%			\caption{\small \label{fig:bipartite} Variation of test accuracy (higher is better) for GCN and LCN- with variation in the graph structure. GCN fails to perform as the number color fraction increases. Refer Section \ref{sec:motivation} for more details. }
%			\label{fig:bipartite}
%		\end{minipage}
%	\end{tabular}
%\vspace{-3mm}
%\end{figure*}

\begin{figure}[t]
	\centering
	\includegraphics[width=\columnwidth]{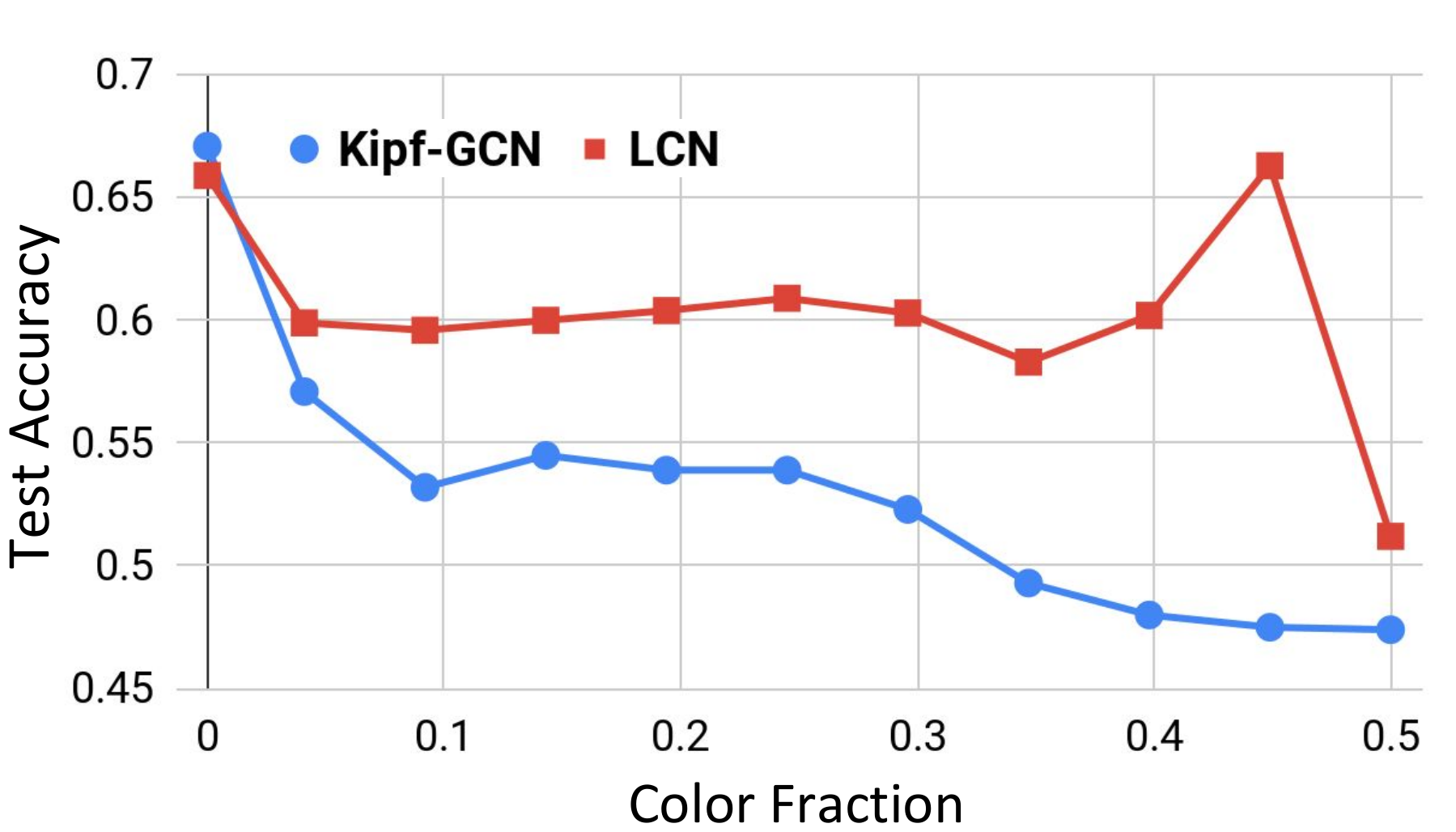}\\
	\caption{\small \label{fig:bipartite} Variation of test accuracy (higher is better) for GCN and LCN- with variation in the graph structure. GCN fails to perform as the number color fraction increases. Refer Section \ref{sec:motivation} for more details. }
\end{figure}

In this section we present a motivating example to demonstrate the use of the \lov{} orthogonal embeddings in the semi-supervised graph transduction task. In particular, we want to show how the embeddings learnt using the \lov{} kernel results in improved accuracy as a parameter called \emph{coloring fraction}, which we define below, varies. To illustrate our hypothesis, we consider a bipartite graph as input to the problem. The reason for this choice is that bipartite graphs are perfect and hence optimal coloring of both the graph $G$ and its complement  $\bar{G}$ (which is also perfect by the perfect graph theorem \citep{Chudnovsky2006}) are  easy to compute in polynomial time. Before explaining the experiment, we start with the following definition.  \\
\textbf{Coloring Fraction:} Given a graph $G = (V, E)$, consider the optimal coloring of the complement graph $\bar{G}$. According to this coloring scheme of the nodes, let $n_d$ represent the number of edges in $G$ such that the pair of nodes each edge connects have different colors. And let $n_t$ represent the total number of pairs of nodes in $G$ such that the nodes in each pair have different colors. Then the \emph{coloring fraction} is defined as $n_d/n_t$.\\
% The coloring fraction of a graph $G(V,E)$ is the ratio of the number of edges in $G$ with different colors to the number of pairs of nodes with different color, where the coloring is optimal with respect to the complement graph $\bar{G}$. \\
As an example, for a complete bipartite graph $G = K(n,n)$ on $2n$ nodes, the complement graph is the union of $2$ disjoint cliques of $n$ nodes each and hence the graph can be colored using $n$ colors. The coloring fraction  is then $\frac{n(n-1)}{2(n-1)(n)} = 0.5.$ The following proposition establishes how coloring fraction varies with removal of edges from a graph. 

\begin{prop} Let $G(V,E)$ be a graph where $\chi(\bar{G})$ is the chromatic number of the complement of $G$. Let $\beta(G)$ be the coloring fraction of $G$. Let $G'$ be the graph obtained from $G$ by removing a set of edges whose nodes have different colors with respect to the optimal coloring of $\bar{G}$.  Then $\chi(\bar{G'}) = \chi(\bar{G})$ whereas $\beta(G') < \beta(G)$.
\end{prop}
\begin{proof}
It should be observed that the optimal coloring for $\bar{G}$ is also a \emph{valid} coloring of $\bar{G'}$ as the edges removed from $G$ are only from nodes with different colors with respect to coloring of $\bar{G}$. To see why it is also an \emph{optimal} coloring, we use contradiction. If there exists a coloring of $\bar{G'}$ with strictly smaller number of colors than $\chi(\bar{G})$, then we can remove edges to form $\bar{G'}$ to obtain $\bar{G}$ such that it is also a valid coloring of $\bar{G}$ as removing edges does not affect the validity of a coloring. However, this contradicts the optimality of the original coloring for $\bar{G}$. Thus $\chi(\bar{G}) = \chi(\bar{G'})$. Moreover, since we are removing edges from $G$, the coloring fraction increases by definition and hence $\beta(G') < \beta(G)$.
\end{proof}
The above proposition says that by removing edges carefully, a local property of the graph (coloring fraction) changes whereas a global property (chromatic number of complement graph) does not change. If the labels of nodes depends on the global property of the graph, then a natural question of interest is to study the sensitivity of algorithms to change in the local property while keeping the global property fixed. This is precisely what we do as we explain below.\\
We begin with a complete bipartite graph $K(n,n)$ whose coloring fraction as computed above is $0.5$. We remove $m$ edges in each step where the nodes of removed edges have different colors (w.r.t optimal coloring of $\bar{G}$).  In each case, the labels are assigned such that nodes with half the colors are assigned to class $0$ and remaining to class $1$. We compute the accuracy of a Laplacian based GCN model vs the proposed LCN model. In our experiment we set $n=50$ and $m=250$. The results averaged over $10$ random splits of $20\%-20\%-60\%$ train-validation-test are presented in Figure \ref{fig:bipartite}. It is clear that as the color fraction increases, the accuracy of the standard GCN drops while that of \lov{} does not. This is because the standard GCN depends on local connectivity property of the graph whereas the orthogonal labeling is done in accordance to the global coloring of the complement graph and is better captured by the proposed LCN.  \\
The above example motivates our study of \lov{} based embeddings in cases where the global structure of the graph is related to the class labels. With this motivation, we propose the \lov{} convolution network in the following section.

\vspace{-3mm}
\section{LCN: Proposed Model }
\vspace{-3mm}
\label{sec:model}
%\begin{figure*}[t]
%	\centering
%	\includegraphics[scale=0.5]{figures/model-crop.pdf}\\
%	\caption{\small Proposed Architecture of \lov{} Convolutional Network along with representative graphs and their embeddings learnt for semi-supervised classification }
%	\label{fig:model}
%\end{figure*}

\begin{small}
 \begin{algorithm}[tb]
	\KwIn{$\mat{A}$, Adjacency matrix of Graph $G$}
	\KwOut{$\mat{K}$: Lovasz Kernel}
	$\begin{aligned}
	[SDP] \mat{Y} \gets & {\text{minimize}} \,\,\, t, \,\,\, \text{subject to:} \\
	&  \mat{Y} \succeq 0, \;\mat{Y}_{ij} = -1, \; \forall(i,j) \notin E, \; \mat{Y}_{ii} = t-1 %\ni a_{ij} = 0
	\end{aligned}$

	$\mat{P} \in \mathbb{R}^{n \times n} \gets Cholesky(\mat{Y})$\;
	\If{$rank(\mat{P})<n$}{
		$c \gets$ random basis element from $Null(\mat{P})$\;
		$u_i = \frac{c+p_i}{\sqrt{t}}$\,  where $p_i$ is $i^{\text{th}}$ column of $\mat{P}$;
	}
	\If{$rank(\mat{P})=n$}{
		$p_i = [p_i \; 0] \in \mathbb{R}^{n+1} \; \forall\; i \in \{1,2, \ldots, n\}$\;
		$u_i = \frac{e_{n+1}+p_i}{\sqrt{t}}$ where $e_{n+1} \in \mathbb{R}^{n+1}$ is the standard basis element\;
	}
	$\mat{U} = [u_1 u_2 \ldots u_n]$\;
	$\mat{K} = \mat{U}^\top\mat{U}$\;
	
	\caption{Lovasz Kernel Matrix Computation}
	\label{alg:lovasz}
\end{algorithm}
\vspace{-1.5mm}	
\end{small}

\begin{figure*}[t!]
	\centering
	\includegraphics[width=\textwidth]{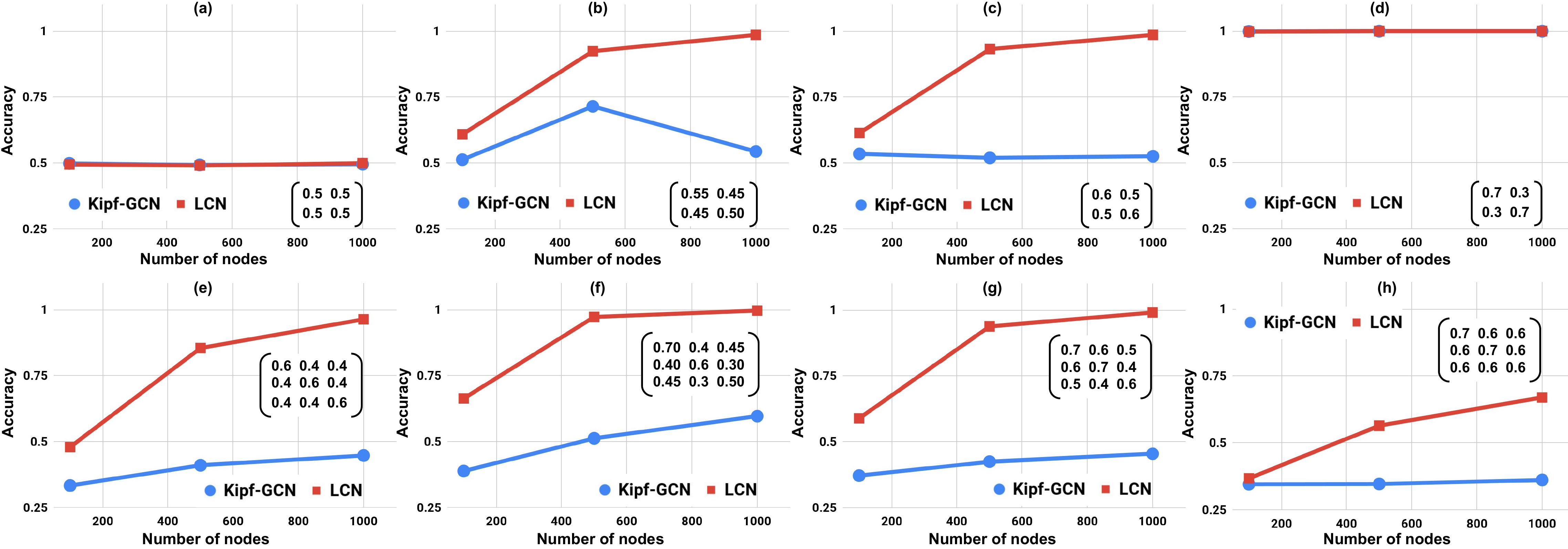}\\
	\caption{\small Test accuracy plots for various synthetically generated graphs from stochastic block model. The matrix in plot denotes the connection probabilities between classes.}
	\label{fig:sbm_results} 
\end{figure*}

In this section, we present our proposed method, the \lov{} Convolution Network (LCN), for semi-supervised graph transduction. As motivated in the previous section, when the class labels depend on the coloring (a global property) of the given graph, it is natural to start training a graph based convolution network which incorporates this property into learning. 
Let $Lab(G)$, as defined in Section \ref{sec:prelim}, represent the set of all possible orthonormal embeddings for a given graph $G$. The set of graph kernel matrices is defined as
\begin{equation*}
 \mathcal{K}(G) \coloneqq \{ \mat{K} \in \mathcal{S}_{n}^{+} | \mat{K}_{ii} = 1, \forall i \in [n]; \mat{K}_{ij} =0. \forall(i,j) \notin E\},
\end{equation*}
where $\mathcal{S}_{n}^{+}$ is the set of all positive semidefinite matrices. \citet{Jethava2013} showed the equivalence between $Lab(G)$ and $\mathcal{K}(G)$. Since $\mat{K} \in \mathcal{K}(G)$ is positive semidefinite, there exists a $\mat{U} \in \mathbb{R}^{d\times n}$ such that $\mat{K} = \mat{U}^\top\mat{U}$. It should be noted that $\mat{K}_{ij} = \vec{u}_{i}^{\top} \vec{u}_{j}$, where $\vec{u}_i$ is the $i$-th column of $\mat{U}$, which implies $\mat{U} \in Lab(G)$. Similarly, it can be shown that for any $\mat{U} \in Lab(G)$, $\mat{K} = \mat{U}^\top\mat{U} \in \mathcal{K}(G)$. 
Given a graph $G$, we follow the procedure described in \cite[Proposition 9.2.9]{lovasz99} for computing the \lov{} orthonormal embedding $\mat{U}$ and the associated kernel matrix $\mat{K}$ optimally. The procedure is summarized in Algorithm \ref{alg:lovasz}. Similar to the normalized Laplacian of a graph, the kernel matrix is also positive semidefinite. 

The kernel computation explained in Algorithm \ref{alg:lovasz} requires solving a Semi Definite Program (SDP), the computational complexity of which is $O(n^6)$. This becomes a huge bottle-neck for large scale datasets. 
Therefore, for large scale datasets, we exploit the following characterization of $\vartheta(G)$  given by \citet{luz2005convex}:

\begin{theorem}[\cite{luz2005convex}]
\label{thm:luz2005}
 For a graph $G = (V, E)$ with $|V| = n$, and let $\mat{C} \in \mathbb{R}^{n \times n}$ be any non-null symmetric matrix with $\mat{C}_{ij} = 0$ whenever $(i,j) \notin E$. Then, 
 \begin{eqnarray*}
\vartheta(G) = \min\limits_{\mat{C}} \nu(G,\mat{C}), \quad \text{where} \\
\nu(G,\mat{C}) = \max\limits_{\vec{x} \geq 0} 2 \vec{x}^{\top}\vec{e} - \vec{x}^{\top} (\frac{\mat{C}}{-\lambda_{\min}(\mat{C})} + \mat{I}) \vec{x} ,
 \end{eqnarray*}
 where $\vec{e} = [1,1,\ldots,1]^{\top}$ and $\lambda_{\min}(\mat{C})$ is the minimum eigen value of $\mat{C}$.

\end{theorem}

Note that the matrix $\mat{K}_{LS} = \frac{\mat{A}}{-\lambda_{\min}(\mat{A})} + \mat{I}$ obtained by fixing $\mat{C} = \mat{A}$ in Theorem \ref{thm:luz2005} is positive semidefinite. Therefore, there exists a labeling $\mat{U} \in \mathbb{R}^{d \times n}$ such that $\mat{U}^{\top}\mat{U} = \mat{K}_{LS}$, which is referred to as LS labeling \citep{Jethava2013}. From Theorem \ref{thm:luz2005}, for any graph $G$, we have 
\begin{equation*}
\vartheta(G) \leq \nu(G, \mat{A}),
\end{equation*}
 an upper bound on $\vartheta(G)$, and the equality holds for a class of graphs called $\mathcal{Q}$ graphs \citep{luz1995upper}. 
 Computation of $\mat{K}_{LS}$ has a complexity of only $O(n^3)$, hence for large scale datasets we approximate the \lov{} kernel by $\mat{K} = \mat{K}_{LS}$. 
%For large scale datasets, we propose to use the kernel matrix associated with \emph{LS labeling} introduced by \citet{luz2005convex} and further analyzed in \citep{Jethava2013}.

We propose to use the following two layered architecture for the problem of semi-supervised classification,
\begin{equation}
\label{eqn:lcn}
 f(\mat{X}, \mat{K}) = \mathrm{softmax} ( \mat{K} \text{~} \text{ReLU}(\mat{K}\mat{X}\mat{W}^{(0)})\mat{W}^{(1)}).
\end{equation}
Similar to GCN, we minimize the cross-entropy loss given in Equation \eqref{eqn:multi-loss} for semi-supervised multi-class classification. We use batch gradient descent for learning the weights $\mat{W}^{(0)}$ and $\mat{W}^{(1)}$. \\
We note that when the class labels are a non-linear mapping of the optimal coloring of $\bar{G}$, LCN with \lov{} kernel $\mathcal{K}(G)$ tunes the weights of the network to learn the mapping.

% !TeX spellcheck = en_US

\section{Experimental Results}
\vspace*{-2mm}

In this section, we report the results of our experiments on several synthetic and real world datasets. We demonstrate the usefulness of the embeddings learnt using the \lov{} convolution networks over several state of the art methods including GCNs, SPORE, normalized and unnormalized laplacian based regularization along with other embeddings such as KS labelings that are described in \citep{Shivanna2015}. We demonstrate our results on Stochastic block models, real world MNIST datasets (binary and multiclass) and several real world UCI datasets. We also run experiments on large scale real world datasets Citeseer, Cora and Pubmed which are standard in GCN literature. In addition to this, we test the goodness of \lov{} based embeddings  in certain perfect graphs called caveman graphs which have been used to model simple social network communities. In addition to simple graphs, we also experiment with hypergraphs with clique expansion to see how the proposed method performs.

% \begin{table}[tbh]
% 	\centering
% 	\scriptsize
% 	\caption{Binary Classification with Random label-to-color assignment\label{tbl:binary}}
% 	\begin{tabular}{cccccccc}
% 		\toprule
% 		Type of Graph & Size & Colour Fraction & (Train, Val, Test) & Adj-GCN & Adj+I-GCN& Lap-GCN & Lov-GCN \\
% 		\midrule
% 		\multirow{8}{20pt}{Caveman Graph} & \multirow{3}{10pt}{(20,5)} & \multirow{3}{10pt}{0.4} & (0.2, 0.2, 0.6) & 0.65 &0.63 &0.65 & {\bf 0.77} \\
% 		&			      & 		& (0.3, 0.2, 0.5) & 0.62 &0.62 & 0.62& {\bf 0.64}\\ 
% 		& 				& 	&(0.5, 0.2, 0.3) & 0.80& 0.77& 0.80&{\bf 0.80} \\
% 		& 				& &	(0.6, 0.2, 0.2) & 0.80 & 0.80& 0.80& {\bf 0.80} \\ 
% 		\cline{2-8}
% 		& \multirow{3}{10pt}{(100,5)}& \multirow{3}{10pt}{0.4}&(0.2, 0.2, 0.6) & 0.65& 0.65& {\bf 0.66}& 0.62\\
% 		&			      & &(0.3, 0.2, 0.5) & 0.73 & 0.75& 0.75& {\bf 0.76}\\ 
% 		& 				& &(0.5, 0.2, 0.3) & 0.82 & 0.85& 0.85& {\bf 0.89}\\
% 		& 				& &(0.6, 0.2, 0.2) & 0.83& 0.84& 0.84 & {\bf 0.91}\\
% 		
% 		\bottomrule	
% 	\end{tabular}
% \end{table}
% \FloatBarrier
% 
\textbf{Stochastic Block Model: Synthetic Data Experiments:} We start by describing our experiments on synthetically generated stochastic block model graphs. We perform the experiment on binary as well as three class classification problem. We report several settings of inter cluster and intra cluster probabilities in Figure \ref{fig:sbm_results} and corresponding embeddings in Figure \ref{fig:embeddings}. In each of the experiments, we varied the number of nodes from $100$ to $1000$ and used a $20\%-10\%-70\%$ train-validation-test split where we use early stopping during training \citep{gcniclr17}. The test accuracy is compared against the standard GCNs (denoted by Kipf-GCN). We make several observations from the results in Figure \ref{fig:sbm_results}. Firstly, as the inter and intra cluster probabilities get closer, it becomes much harder for GCN to classify well whereas LCN outperforms GCN by a significant margin. Secondly, as the size of the graph increases, the differences in connections become more critical and this is reflected in the increased accuracy with increase  in nodes for LCN, whereas accuracy of GCN is almost agnostic to the number of nodes. Finally, as the graph becomes denser i.e., as connection probabilities tend towards $1$, LCN performs much better than GCN in the three class setting. These results demonstrate the advantage of using LCNs over GCNs for semi-supervised classification tasks for SBM models.

\begin{figure*}[tb]
	
	\begin{subfigure}{.5\textwidth}
		\centering
		\includegraphics[width=0.95\linewidth]{./figures/embed_aistats/sbm5545}
		{\small (b)}
	\end{subfigure}
	\begin{subfigure}{.5\textwidth}
		\centering
		\includegraphics[width=0.95\linewidth]{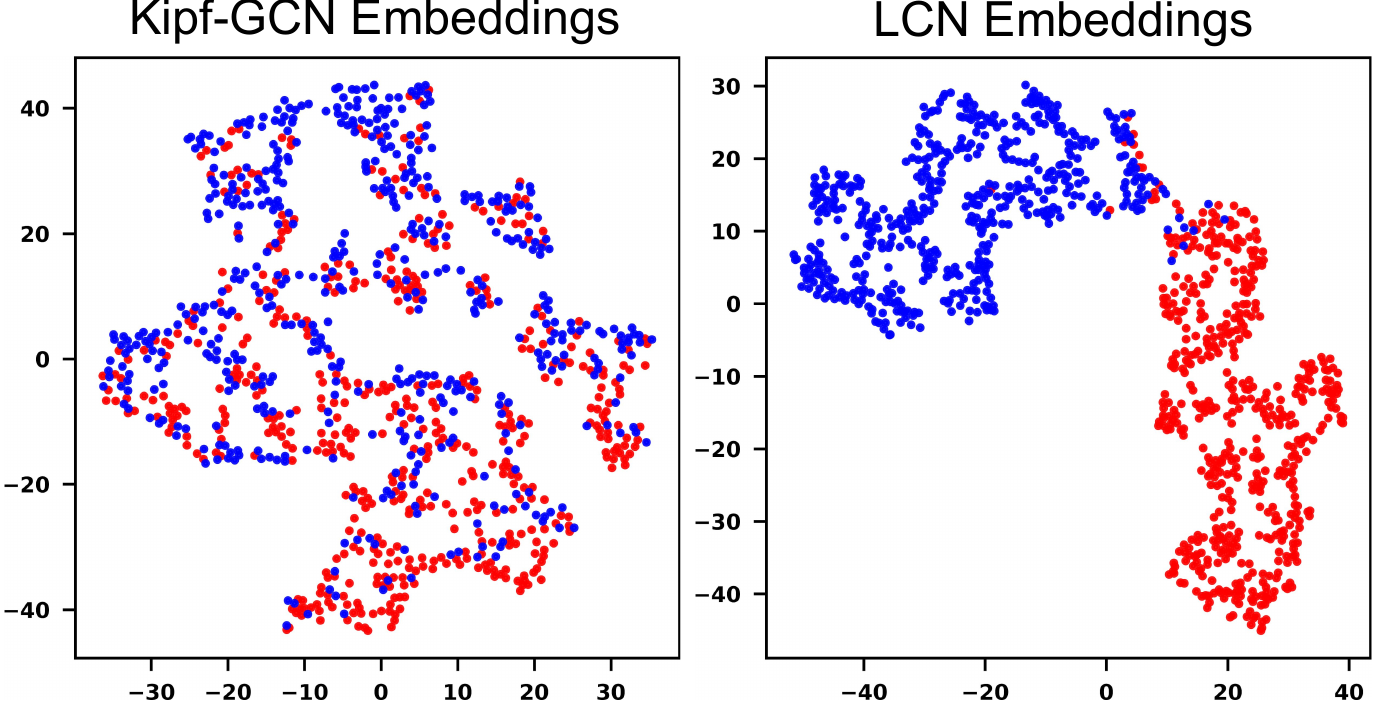}
		{\small (c)}
	\end{subfigure}
	\begin{subfigure}{.5\textwidth}
		\centering
		\includegraphics[width=0.95\linewidth]{./figures/embed_aistats/sbm6444}
		{\small (e) }
	\end{subfigure}
	\begin{subfigure}{.5\textwidth}
		\centering
		\includegraphics[width=0.95\linewidth]{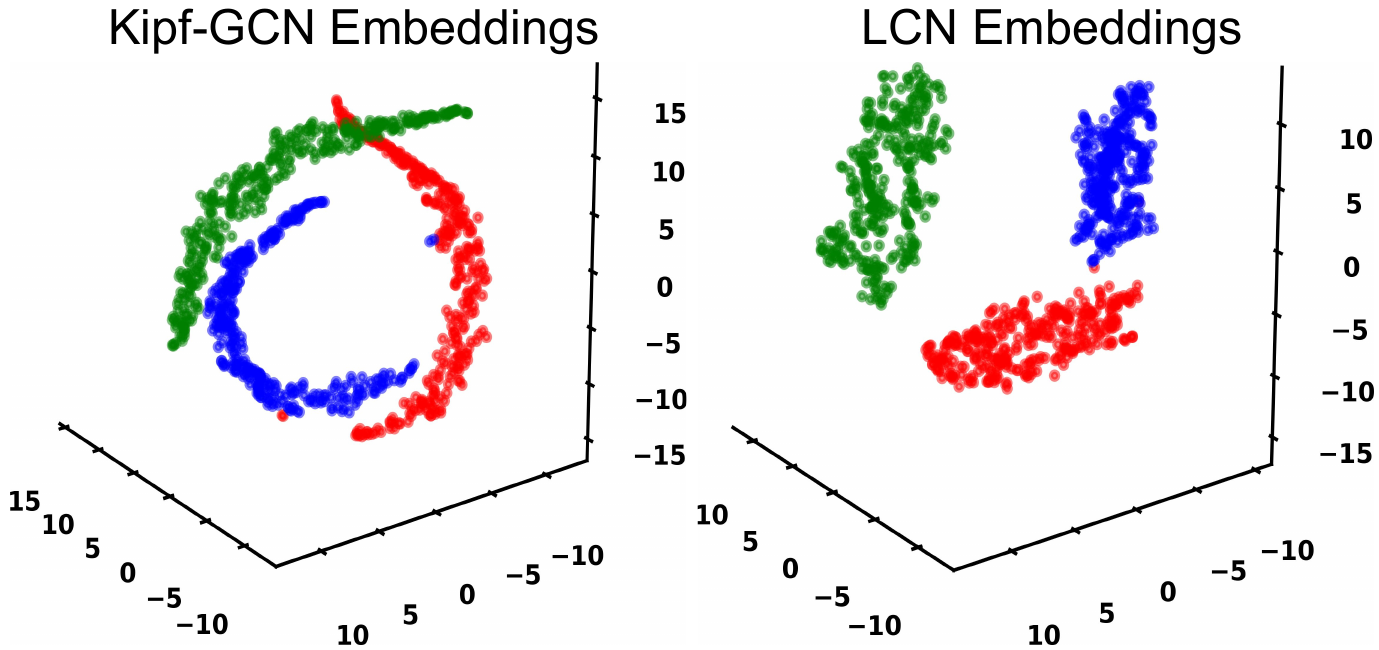} \\
		{\small (f)}
	\end{subfigure}
	\caption{\label{fig:embeddings}Embeddings learnt for settings corresponding to Figure \ref{fig:sbm_results} (b), (c), (e), (f) for n=1000}
	\vspace{-3mm}
\end{figure*}
\textbf{Real World Data Experiments:} We run several experiments on real world datasets including MNIST and UCI datasets. To make a fair comparison with state of the art, we first run the same set of binary classification experiments as in \citep{Shivanna2015} and compare it with GCNs \citep{gcniclr17}, Graph Partition Neural Networks (GPNN) \citep{liao2018graph} and our proposed LCNs. These include experiments on $6$ UCI datasets (breast-cancer $n=683$, diabetes $n=768$, fourclass $n=862$, heart $n=270$, ionosphere $n=351$ and sonar $n=208$) and experiments on certain pair of classes from a subsampled set of images from the MNIST datasets. Table \ref{tab:mnist2} reports the results for these experiments with various input embeddings including Laplacian (normalized, unnormalized), KS embedding and others as considered in \citep{Shivanna2015}. For MNIST datasets, the results are averaged over five randomly sampled graphs and for UCI datasets, the results are averaged over five random splits. As can be seen, LCN performs significantly better than SPORE and performs much better than GCNs in all datasets except two. In addition to the binary classification experiment, we also conducted three class  classification on $500$ and $2000$ images from MNIST (randomly subsampled from classes $1$, $2$ and $7$) and also $10$ class classification where we randomly subsample 2000 images from all classes. The results are reported in Table \ref{tab:mnist3} As the classes increase, LCN significantly outperforms other GCN baselines.  

To be consistent with the GCN literature, we also run experiments on large scale datasets Citeseer, Cora  and Pubmed. All these datasets are citation networks, where each document is represented as a node in the graph with an edge between nodes indicating the citation relation. The aim is to classify the documents into one of the predefined classes. We use the same splits as in \citep{planetoid_icml16}. Table \ref{tbl:large_datasets} shows the results on these large scale datasets, as explained in Section \ref{sec:model} we use the approximate $K_{LS}$ kernel for these datasets and LCN (LS) refers to this setting. LCN outperforms other state-of-the-art baselines on all three datasets, Citeseer, Cora and Pubmed. Node2vec \citep{grover2016node2vec} is an unsupervised method for learning node representations for a given graph using just the structure of the graph. In table \ref{tbl:large_datasets}, \emph{Node2vec} refers to the model when the kernel is obtained from normalized Node2vec embeddings, which achieves a  significantly poor performance.
%We adopt the experimental setting from \citep{} and then ran our models on all of their datasets, and additionally we also show the 10 class classification results for our model which beat their best two class classification results. We demonstrate these number as a sanity check for our model.\\

\begin{table*}[tb]
	\centering
%	\scriptsize
		
		\begin{tabular}{lccccccc}
			\toprule
			Dataset & Un-Lap        & N-Lap & KS    & SPORE & Kipf-GCN       & GPNN         & LCN \\ 
			\midrule
			breast-cancer & 88.2 & 93.3 & 92.8 & 96.7 & \textbf{97.6} & 95.5 &  97.2 \\
			diabetes & 68.9 & 69.3 & 69.4 & 73.3 & 71.4 & 68.0 &  \textbf{76.3} \\
			fourclass & 70.0 & 70.0 & 70.4 & 78.0 & 80.5 & 73.9 &  \textbf{81.7} \\
			heart & 72.0 & 75.6 & 76.4 & 82.0 & \textbf{85.1 }& 81.1 &  82.5 \\
			ionosphere & 67.8 & 68.0 & 68.1 & 76.1 & 76.1 & 70.0 &  \textbf{87.9} \\
			sonar & 58.8 & 59.0 & 59.3 & 63.9 & 71.4 & 64.8 &  \textbf{73.2} \\
			\midrule
			mnist-500 1 vs 2 & 75.6 & 80.6 & 79.7 & 85.8 & 98.0 & 96.2 &  \textbf{99.0} \\
			mnist-500 3 vs 8 & 76.9 & 81.9 & 83.3 & 86.1 & 92.3 & 83.1 &  \textbf{93.7} \\
			mnist-500 4 vs 9 & 68.4 & 72.0 & 72.2 & 74.9 & \textbf{89.4} & 88.5 &  83.3 \\
			\midrule
			mnist-2000 1 vs 2 & 83.8 & 96.2 & 95.0 & 96.7 & 99.0 & 97.5 &  \textbf{99.2} \\
			mnist-2000 3 vs 8 & 55.2 & 87.4 & 87.4 & 91.4 & 94.7 & 89.6 &  \textbf{95.7} \\
			mnist-2000 1 vs 7 & 90.7 & 96.8 & 96.6 & 97.3 & \textbf{98.8} & 96.4 &  98.7 \\
			\bottomrule
		\end{tabular}
	\caption{\label{tab:mnist2} Binary Classification with Random label-to-color assignment in UCI and MNIST datasets.}
\end{table*}

%\begin{table*}[tbh]
%	\centering
%		\scriptsize
%	\caption{Binary Classification with Random label-to-color assignment in UCI and MNIST datasets.}
%	\label{tab:mnist2}
%	\begin{tabular}{lcccccccccccc}
%		\toprule
%		Dataset & breast-cancer & diabetes & fourclass & heart & ionosphere & sonar & mnist-500 1/2 & mnist-500 3/8 & mnist-500 4/9 & mnist-2000 1/2 & mnist-2000 3/8 & mnist-2000 1/7 \\
%		\midrule
%		Un-Lap & 88.22 & 68.89 & 70 & 71.97 & 67.77 & 58.81 & 75.55 & 76.88 & 68.44 & 83.8 & 55.15 & 90.65 \\
%		N-Lap & 93.33 & 69.33 & 70 & 75.56 & 68 & 58.97 & 80.55 & 81.88 & 72 & 96.23 & 87.35 & 96.8 \\
%		KS & 92.77 & 69.44 & 70.44 & 76.42 & 68.11 & 59.29 & 79.66 & 83.33 & 72.22 & 94.95 & 87.35 & 96.55 \\
%		SPORE & 96.67 & 73.33 & 78 & 81.97 & 76.11 & 63.92 & 85.77 & 86.11 & 74.88 & 96.72 & 91.35 & 97.25 \\
%		Kipf-GCN & 96 & 66.51 & 76.38 & 80.92 & 73.92 & 65 & 97.7 & 87.42 & 71.36 & 97.82 & 88.43 & 97.68 \\
%		LCN & 96.41 & 71.82 & 79.44 & 75.4 & 76.34 & 62.96 & 97.83 & 89.49 & 79.75 & 98.72 & 93.68 & 97.92 \\
%		KLS-GCN & 96.8 & 72.69 & 76.77 & 79.34 & 78.9 & 67.8 & 98.43 & 92.86 & 90.67 & 98.81 & 95.82 & 98.52 \\
%		\bottomrule
%	\end{tabular}
%\end{table*}

\begin{table}[tb]
%		\parbox{.60\linewidth}{
		\centering
		\small
		\begin{tabular}{lccc}
			\toprule
			Dataset &  Kipf-GCN  & GPNN & LCN\\% & A+I-GCN& KLS-GCN \\
			\midrule
			%mnist-2000 4 vs 9 & 74.46  &  89.98&{\bf 93.65} \\%& 91.15& 93.65(50-hid)\\
			%\midrule
			mnist 500 127  & 96.1  &  93.8  &  \textbf{97.5} \\
			mnist 2000 127 & 97.2  &  94.4  &  \textbf{97.4} \\
			mnist all 2000 & 84.4  &  56.9  &  \textbf{85.1} \\
			
			\bottomrule	
		\end{tabular}
	\caption{\label{tab:mnist3} Multi class Classification in MNIST dataset \label{tbl:spore2}}
\end{table}

%\parbox{.40\linewidth}{
\begin{table}[tb]
	\centering
	\small
	\begin{tabular}{lcccc}
		\toprule
		Dataset & Node2vec    &    Kipf-GCN    &    GPNN &   LCN \\
		\midrule
		Citeseer & 23.1  &  70.3  &  69.7  &  \textbf{73.5}\\
		Cora & 31.9  &  81.5  &  81.8  &  \textbf{82.6} \\
		Pubmed & 42.3  &  79    &  79.3  &  \textbf{79.7} \\
		\bottomrule	
	\end{tabular}
\caption{\label{tbl:large_datasets} Performance for semi-supervised on Citeseer, Cora, Pubmed datasets}

\end{table}

\begin{table}[tb]
	\centering
	\small
	 \begin{tabular}{lccccc}
		\toprule
		$(n,k)$ & Kipf-GCN & LCN  & Avg\_same & Avg\_diff  \\
		\midrule
		(50, 10) & 0.92 & {\bf 0.93}  & 0.83 & -0.008 \\
		(75, 6) & 0.77 & {\bf 0.80}  & 0.80 & -0.005 \\
		(100, 5) & 0.71 & {\bf 0.73}  & 0.79 & -0.003 \\
		(100, 7) & {\bf 0.81} & {\bf 0.81}  & 0.80 & -0.003 \\
		\bottomrule  
	\end{tabular}
	\caption{\label{fig:caveman} Caveman graph experiment: Average test accuracy  of Kipf-GCN and LCN on various caveman graphs.}
\end{table}

%\begin{figure*}[tbh]
%\centering
% \begin{tabular}{cc}
% \begin{minipage}{0.3\hsize}
% \centering
%  
%  \includegraphics[scale=0.25]{figures/caveman.pdf}\\
%  {\small (a) A caveman graph of size (5,5) i.e., 5 cliques of size 5.}
%  \end{minipage}
% \begin{minipage}{0.65\hsize}
%% \begin{table}
%% \caption{\label{tbl:caveman_graph} Test accuracy (higher is better) of Kipf-GCN and LCN on various Caveman graphs. Please see section \ref{sec:caveman} for more details.}
% \centering
% \scriptsize
% {\small (b) Average test accuracy  of Kipf-GCN and LCN on various caveman graphs. }\\
% \begin{tabular}{cccccc}
% \toprule
% $(n,k)$ & Kipf-GCN & LCN  & Avg\_same & Avg\_diff  \\
% \midrule
% (50, 10) & 0.92 & {\bf 0.93}  & 0.83 & -0.008 \\
% (75, 6) & 0.77 & {\bf 0.80}  & 0.80 & -0.005 \\
% (100, 5) & 0.71 & {\bf 0.73}  & 0.79 & -0.003 \\
% (100, 7) & {\bf 0.81} & {\bf 0.81}  & 0.80 & -0.003 \\
% \bottomrule  
% \end{tabular}
% \end{minipage} 
% \end{tabular}
%\caption{\label{fig:caveman}Caveman graph experiment.}
%\end{figure*}

\begin{figure}[tb]
	 \centering
	\includegraphics[width=\columnwidth]{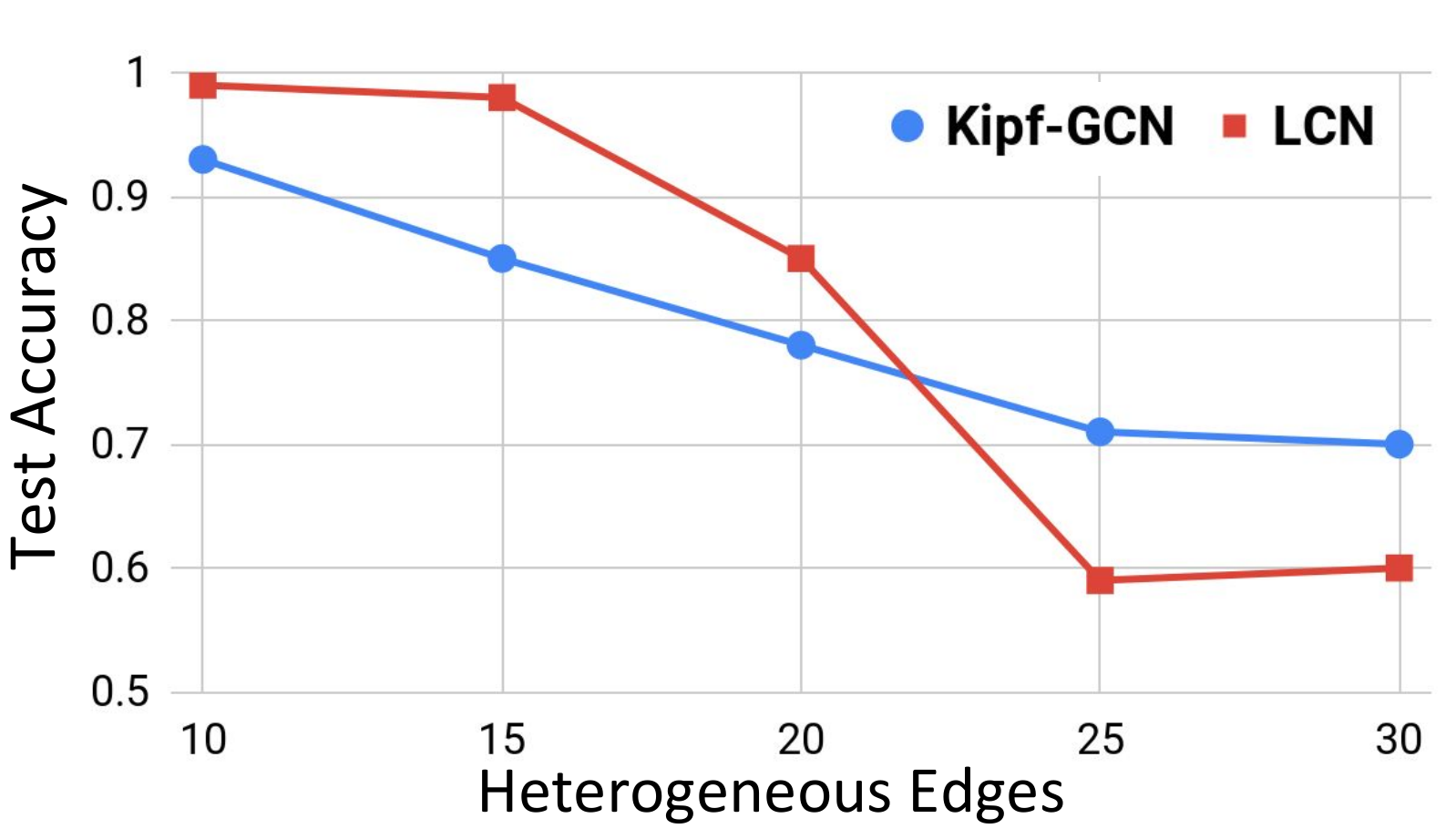}\\	
	\caption{\small (c) Behavior of test accuracy with increase in the heterogeneous edges in the hypergraph.\label{fig:hyper_graph}}
\end{figure}

%\begin{figure*}[!tbh]
% \begin{tabular}{ccc}
%  \begin{minipage}{0.25\hsize}
%   \centering
%\includegraphics[scale=0.23]{figures/hypergraph.pdf}\\
%{\small (a) Schematic view of a hypergraph.}
%  \end{minipage}
%  \begin{minipage}{0.25\hsize}
%   \centering
%   \includegraphics[scale=0.23]{figures/hypergraph_clique.pdf}\\
%   {\small (b) Clique expansion of a hypergraph.}
%  \end{minipage}
%  \begin{minipage}{0.5\hsize}
%  \centering
%   \includegraphics[scale=0.19]{figures/hyper}\\
%   {\small (c) Behavior of test accuracy with increase in the heterogeneous edges in the hypergraph.}
%  \end{minipage}
% \end{tabular}
%\caption{\label{fig:hyper_graph} Hypergraph exiperment.}
%\end{figure*}

%\begin{figure*}[!tbh]
% \begin{tabular}{ccc}
%  \begin{minipage}{0.25\hsize}
%   \centering
%\includegraphics[scale=0.8]{figures/hypergraph.pdf}\\
%{\small (a) Schematic view of a hypergraph.}
%  \end{minipage}
%  \begin{minipage}{0.25\hsize}
%   \centering
%   \includegraphics[scale=0.23]{figures/hypergraph_clique.pdf}\\
%   {\small (b) Clique expansion of a hypergraph.}
%  \end{minipage}
%  \begin{minipage}{0.5\hsize}
%  \centering
%   \includegraphics[scale=0.19]{figures/hyper}\\
%   {\small (c) Behavior of test accuracy with increase in the heterogeneous edges in the hypergraph.}
%  \end{minipage}
% \end{tabular}
%\caption{\label{fig:hyper_graph} Hypergraph exiperment.}
%\end{figure*}

\textbf{Caveman graph: Goodness of Embeddings Experiment:} A connected caveman graph of size $(n,k)$ is formed by modifying a set of isolated $k$-cliques (or \emph{caves}) by removing one edge from each clique and using it to connect to a neighboring clique along a central cycle such that all $n$ cliques form a single unbroken loop \citep{Watts1999}. Caveman graphs are perfect graphs and are used for modeling simple communities in social networks \citep{Kang2011}. We run our experiments on various synthetic caveman graphs. For every caveman graph, we compute the optimal coloring of the complement graph. We consider a binary classification setting and randomly assign nodes corresponding to half of the colors to class $0$ and the other half to the  class $1$. 
We work with a $20\%-20\%-60\%$ train-validation-test splits, for every graph we run the experiments on 10 random splits. Table \ref{fig:caveman} shows the test accuracy of Kipf-GCN and LCN. In Table \ref{fig:caveman}, \emph{Avg\_same} stands for average inner product of the representations of the nodes with same color and \emph{Avg\_diff} stands for that of the nodes with different colors. As we see, LCN performs better on all cases considered. Also, the average dot products of nodes with same color is high and those with different colors is close to zero showing that the representations are as well separated as possible for nodes with different colors. 

\textbf{Hypergraphs: Homogeneous vs Heterogeneous Edges Experiment:} Though our main focus is on simple graphs, we also experiment with synthetic hypergraphs. A hypergraph is a generalized version of a graph where a hyper edge consists of a set of nodes.  However, for hypergraphs, to the best of our knowledge, orthogonal embeddings and \lov{} theta function are not defined. Therefore, we consider the clique expansion of the hypergraphs \citep{lhg06}. Clique expansion creates a simple graph from a hypergraph by replacing every hyperedge with a clique. In our experiments, we generated a hypergraph of 100 nodes with every hyperedge containing 35 nodes. We consider a binary classification setting and assign randomly 50 nodes to one class and the other 50 to a different class. We randomly create 20 hyperedges such that all the nodes in the hyperedge belong to same class, we call these edges homogeneous edges. We also create $m$ random hyperedges such that the label distribution of the nodes in the hyperedge is 2:3, we call these edges heterogeneous. We vary $m$ between 10 and 30 and create multiple hypergraphs. We work with a 20\%-20\%-60\% train-validation-test split and average across ten runs per each hypergraph. Figure \ref{fig:hyper_graph} shows the behavior of test accuracy with increase in the number of heterogeneous edges.  As one can see LCN performs much better than GCN when the number of heterogeneous edges are small (and hence the clique expansion has a community-like structure) whereas GCNs tend to perform better with increase in the number of heterogeneous edges. 

\vspace{-3mm}
\section{Conclusion}
\vspace{-3mm}
\label{sec:conclusion}
We propose \lov{} Convolution Networks for the problem of semi supervised learning on graphs. Our analysis shows settings where LCNs perform much better than GCNs. Our results on real world and synthetic datasets demonstrate the superior embeddings learnt by LCNs and show that they significantly outperform GCNs. Future work includes detailed analysis of \lov{} embeddings for hypergraphs, use of LCNs for community detection and clustering. 
\bibliographystyle{apalike}

\bibliography{references}

\begin{thebibliography}{}

\bibitem[Agarwal, 2006]{shivani_icml06}
Agarwal, S. (2006).
\newblock Ranking on graph data.
\newblock In {\em International Conference on Machine Learning}.

\bibitem[Ando and Zhang, 2007]{Ando2007}
Ando, R.~K. and Zhang, T. (2007).
\newblock Learning on graph with laplacian regularization.
\newblock In {\em Advances in neural information processing systems}.

\bibitem[Atwood and Towsley, 2016]{dcnn_nips16}
Atwood, J. and Towsley, D. (2016).
\newblock Diffusion-convolutional neural networks.
\newblock In {\em Advances in neural information processing systems}.

\bibitem[Auer et~al., 2007]{dbpedia07}
Auer, S., Bizer, C., Kobilarov, G., Lehmann, J., Cyganiak, R., and Ives, Z.
  (2007).
\newblock Dbpedia: A nucleus for a web of open data.
\newblock In {\em Proceedings of the 6th International The Semantic Web and 2Nd
  Asian Conference on Asian Semantic Web Conference}.

\bibitem[Belkin et~al., 2006]{sslintro06}
Belkin, M., Niyogi, P., and Sindhwani, V. (2006).
\newblock Manifold regularization: A geometric framework for learning from
  labeled and unlabeled examples.
\newblock {\em JMLR}.

\bibitem[Bollacker et~al., 2008]{freebase08}
Bollacker, K., Evans, C., Paritosh, P., Sturge, T., and Taylor, J. (2008).
\newblock Freebase: A collaboratively created graph database for structuring
  human knowledge.
\newblock In {\em Proceedings of the 2008 ACM SIGMOD International Conference
  on Management of Data}.

\bibitem[Bruna et~al., 2014]{gcn_iclr14}
Bruna, J., Zaremba, W., Szlam, A., and LeCun, Y. (2014).
\newblock Spectral networks and locally connected networks on graphs.
\newblock In {\em ICLR}.

\bibitem[Chen et~al., 2018]{Chen2018}
Chen, J., Ma, T., and Xiao, C. (2018).
\newblock Fastgcn: Fast learning with graph convolutional networks via
  importance sampling.
\newblock {\em arXiv preprint arXiv:1801.10247}.

\bibitem[Chudnovsky et~al., 2006]{Chudnovsky2006}
Chudnovsky, M., Robertson, N., Seymour, P., and Thomas, R. (2006).
\newblock The strong perfect graph theorem.
\newblock {\em Annals of mathematics}.

\bibitem[Condon and Karp, 1999]{condon1999algorithms}
Condon, A. and Karp, R.~M. (1999).
\newblock Algorithms for graph partitioning on the planted partition model.
\newblock In {\em Randomization, Approximation, and Combinatorial Optimization.
  Algorithms and Techniques}.

\bibitem[Defferrard et~al., 2016]{chebnet_nips16}
Defferrard, M., Bresson, X., and Vandergheynst, P. (2016).
\newblock Convolutional neural networks on graphs with fast localized spectral
  filtering.
\newblock In {\em Advances in neural information processing systems}.

\bibitem[Duvenaud et~al., 2015]{gcn_nips15}
Duvenaud, D.~K., Maclaurin, D., Iparraguirre, J., Bombarell, R., Hirzel, T.,
  Aspuru-Guzik, A., and Adams, R.~P. (2015).
\newblock Convolutional networks on graphs for learning molecular fingerprints.
\newblock In {\em Advances in neural information processing systems}.

\bibitem[Getoor and Taskar, 2007]{getoor2007introduction}
Getoor, L. and Taskar, B. (2007).
\newblock {\em Introduction to statistical relational learning}.

\bibitem[Giles et~al., 1998]{citeseer98}
Giles, C.~L., Bollacker, K.~D., and Lawrence, S. (1998).
\newblock Citeseer: An automatic citation indexing system.
\newblock In {\em Proceedings of the Third ACM Conference on Digital
  Libraries}.

\bibitem[Grover and Leskovec, 2016]{grover2016node2vec}
Grover, A. and Leskovec, J. (2016).
\newblock node2vec: Scalable feature learning for networks.
\newblock In {\em Proceedings of the 22nd ACM SIGKDD international conference
  on Knowledge discovery and data mining}, pages 855--864. ACM.

\bibitem[Henaff et~al., 2015]{gcn_arxiv15}
Henaff, M., Bruna, J., and LeCun, Y. (2015).
\newblock Deep convolutional networks on graph-structured data.
\newblock {\em CoRR}.

\bibitem[Holland et~al., 1983]{holland1983stochastic}
Holland, P.~W., Laskey, K.~B., and Leinhardt, S. (1983).
\newblock Stochastic blockmodels: First steps.
\newblock {\em Social networks}.

\bibitem[Jain et~al., 2016]{Jain2016}
Jain, A., Zamir, A.~R., Savarese, S., and Saxena, A. (2016).
\newblock Structural-rnn: Deep learning on spatio-temporal graphs.
\newblock In {\em Proceedings of the IEEE Conference on Computer Vision and
  Pattern Recognition}.

\bibitem[Jethava et~al., 2013]{Jethava2013}
Jethava, V., Martinsson, A., Bhattacharyya, C., and Dubhashi, D. (2013).
\newblock Lovasz theta function, svms and finding dense subgraphs.
\newblock {\em Journal of Machine Learning Research}.

\bibitem[Johansson et~al., 2014]{Johansson2014}
Johansson, F., Jethava, V., Dubhashi, D., and Bhattacharyya, C. (2014).
\newblock Global graph kernels using geometric embeddings.
\newblock In {\em Proceedings of the 31st International Conference on Machine
  Learning, ICML 2014, Beijing, China, 21-26 June 2014}.

\bibitem[Kang and Faloutsos, 2011]{Kang2011}
Kang, U. and Faloutsos, C. (2011).
\newblock Beyond'caveman communities': Hubs and spokes for graph compression
  and mining.
\newblock In {\em Data Mining (ICDM), 2011 IEEE 11th International Conference
  on}.

\bibitem[Kipf and Welling, 2017]{gcniclr17}
Kipf, T.~N. and Welling, M. (2017).
\newblock Semi-supervised classification with graph convolutional networks.
\newblock In {\em ICLR}.

\bibitem[Leskovec et~al., 2010a]{sn_www10}
Leskovec, J., Huttenlocher, D., and Kleinberg, J. (2010a).
\newblock Predicting positive and negative links in online social networks.
\newblock In {\em Proceedings of the 19th International Conference on World
  Wide Web}.

\bibitem[Leskovec et~al., 2010b]{sn10}
Leskovec, J., Huttenlocher, D., and Kleinberg, J. (2010b).
\newblock Signed networks in social media.
\newblock In {\em Proceedings of the SIGCHI Conference on Human Factors in
  Computing Systems}.

\bibitem[Li et~al., 2018]{co_self_gcn_aaai18}
Li, Q., Han, Z., and Wu, X.-M. (2018).
\newblock Deeper insights into graph convolutional networks for semi-supervised
  learning.
\newblock In {\em AAAI}.

\bibitem[Liao et~al., 2018]{liao2018graph}
Liao, R., Brockschmidt, M., Tarlow, D., Gaunt, A.~L., Urtasun, R., and Zemel,
  R. (2018).
\newblock Graph partition neural networks for semi-supervised classification.
\newblock {\em arXiv preprint arXiv:1803.06272}.

\bibitem[Lov{\'a}sz, 1979]{Lovasz1979}
Lov{\'a}sz, L. (1979).
\newblock On the shannon capacity of a graph.
\newblock {\em IEEE Transactions on Information theory}.

\bibitem[Lov{\'a}sz, 2009]{lovasz2009characterization}
Lov{\'a}sz, L. (2009).
\newblock A characterization of perfect graphs.
\newblock In {\em Classic Papers in Combinatorics}, pages 447--450. Springer.

\bibitem[Lovász and Vesztergombi, 1999]{lovasz99}
Lovász, L. and Vesztergombi, K. (1999).
\newblock Geometric representations of graphs.
\newblock In {\em IN PAUL ERDÖS, PROC. CONF}.

\bibitem[Lu and Getoor, 2003]{cn03}
Lu, Q. and Getoor, L. (2003).
\newblock Link-based classification.
\newblock In {\em Proceedings of the Twentieth International Conference on
  International Conference on Machine Learning}.

\bibitem[Luz, 1995]{luz1995upper}
Luz, C.~J. (1995).
\newblock An upper bound on the independence number of a graph computable in
  polynomial-time.
\newblock {\em Operations Research Letters}, 18(3):139--145.

\bibitem[Luz and Schrijver, 2005]{luz2005convex}
Luz, C.~J. and Schrijver, A. (2005).
\newblock A convex quadratic characterization of the lov{\'a}sz theta number.
\newblock {\em SIAM Journal on Discrete Mathematics}, 19(2):382--387.

\bibitem[McAuley and Leskovec, 2012]{sn12}
McAuley, J. and Leskovec, J. (2012).
\newblock Learning to discover social circles in ego networks.
\newblock In {\em Proceedings of the 25th International Conference on Neural
  Information Processing Systems - Volume 1}.

\bibitem[Ram{\'\i}rez-Alfons{\'\i}n and Reed, 2001]{ramirez2001perfect}
Ram{\'\i}rez-Alfons{\'\i}n, J.~L. and Reed, B.~A. (2001).
\newblock {\em Perfect graphs}, volume~44.
\newblock Wiley.

\bibitem[Sen et~al., 2008]{ccnd08}
Sen, P., Namata, G.~M., Bilgic, M., Getoor, L., Gallagher, B., and Eliassi-Rad,
  T. (2008).
\newblock Collective classification in network data.
\newblock {\em AI Magazine}.

\bibitem[Shivanna et~al., 2015]{Shivanna2015}
Shivanna, R., Chatterjee, B.~K., Sankaran, R., Bhattacharyya, C., and Bach, F.
  (2015).
\newblock Spectral norm regularization of orthonormal representations for graph
  transduction.
\newblock In {\em Advances in Neural Information Processing Systems 28}. Curran
  Associates, Inc.

\bibitem[Subramanya and Talukdar, 2014]{subramanya2014graph}
Subramanya, A. and Talukdar, P.~P. (2014).
\newblock Graph-based semi-supervised learning.
\newblock {\em Synthesis Lectures on Artificial Intelligence and Machine
  Learning}, 8(4):1--125.

\bibitem[Suchanek et~al., 2007]{yago07}
Suchanek, F.~M., Kasneci, G., and Weikum, G. (2007).
\newblock Yago: A core of semantic knowledge.
\newblock In {\em Proceedings of the 16th International Conference on World
  Wide Web}.

\bibitem[Watts, 1999]{Watts1999}
Watts, D.~J. (1999).
\newblock Networks, dynamics, and the small-world phenomenon.
\newblock {\em American Journal of sociology}.

\bibitem[Yang et~al., 2016]{planetoid_icml16}
Yang, Z., Cohen, W.~W., and Salakhutdinov, R. (2016).
\newblock Revisiting semi-supervised learning with graph embeddings.
\newblock In {\em International Conference on Machine Learning}.

\bibitem[Zhou et~al., 2004]{Zhou2004}
Zhou, D., Bousquet, O., Lal, T.~N., Weston, J., and Sch{\"o}lkopf, B. (2004).
\newblock Learning with local and global consistency.
\newblock In {\em Advances in neural information processing systems}.

\bibitem[Zhou et~al., 2006]{lhg06}
Zhou, D., Huang, J., and Sch{\"{o}}lkopf, B. (2006).
\newblock Learning with hypergraphs: Clustering, classification, and embedding.
\newblock In {\em Advances in neural information processing systems}.

\bibitem[Zhu et~al., 2003]{sslintroicml03}
Zhu, X., Ghahramani, Z., and Lafferty, J. (2003).
\newblock Semi-supervised learning using gaussian fields and harmonic
  functions.
\newblock In {\em International Conference on Machine Learning}.

\bibitem[Zhuang and Ma, 2018]{dual_gcn_www18}
Zhuang, C. and Ma, Q. (2018).
\newblock Dual graph convolutional networks for graph-based semi-supervised
  classification.
\newblock In {\em TheWebConf}.

\bibitem[Zitnik and Leskovec, 2017]{protein17}
Zitnik, M. and Leskovec, J. (2017).
\newblock Predicting multicellular function through multi-layer tissue
  networks.
\newblock {\em Bioinformatics}.

\end{thebibliography}

\end{document}